\documentclass[letterpaper, 10 pt, conference]{ieeeconf}
\usepackage{graphicx}
\usepackage{amsmath, amssymb, mathtools}
\usepackage{glossaries}
\usepackage{algorithm}
\usepackage[algo2e]{algorithm2e}
\usepackage{algorithmic}
\usepackage{booktabs}
\usepackage{caption}
\usepackage{subcaption}
\usepackage{balance}
\usepackage{hyperref}
\usepackage{comment}

\usepackage{amsthm}

\IEEEoverridecommandlockouts                         
\title{\LARGE \bf
Graph Neural Model Predictive Control for High-Dimensional Systems
}

\author{
    Patrick Benito Eberhard$^{1}$, 
    Luis Pabon$^{2}$, 
    Daniele Gammelli$^{2}$, 
    Hugo Buurmeijer$^{2}$,\\
    Amon Lahr$^{1}$,
    Mark Leone$^{2}$,
    Andrea Carron$^{1}$,
    Marco Pavone$^{2,3}$
    \thanks{$^{1}$Institute for Dynamic Systems and Control, ETH Zürich, Zürich CH-8092, Switzerland (e-mail: \{peberhard, amlahr, carrona\}@ethz.ch).}%
    \thanks{$^{2}$Department of Aeronautics and Astronautics, Stanford University, Stanford, CA (e-mail: \{lpabon, gammelli, hbuurmei, mleone, pavone\}@stanford.edu).}%
    \thanks{$^{3}$NVIDIA Research}
}

\newacronym{mpc}{MPC}{Model Predictive Control}
\newacronym{ocp}{OCP}{Optimal Control Problem}
\newacronym[plural=GNNs, firstplural=Graph Neural Networks (GNNs)]{gnn}{GNN}{Graph Neural Network}
\newacronym{qp}{QP}{Quadratic Program}
\newacronym{dof}{DoF}{Degrees of Freedom}
\newacronym{fem}{FEM}{Finite Element Method}
\newacronym[plural=SSMs, firstplural=Spectral Submanifolds (SSMs)]{ssm}{SSM}{Spectral Submanifold}
\newacronym{admm}{ADMM}{Alternating Direction Method of Multipliers}
\newacronym{gpu}{GPU}{Graphics Processing Unit}
\newacronym{mlp}{MLP}{Multi-Layer Perceptron}
\newacronym[plural=CNNs, firstplural=Convolutional Neural Networks (CNNs)]{cnn}{CNN}{Convolutional Neural Network}
\newacronym{relu}{ReLU}{Rectified Linear Unit}
\newacronym{ipm}{IPM}{Interior Point Method}
\newacronym{kkt}{KKT}{Karush-Kuhn-Tucker}
\newacronym{licq}{LICQ}{Linear Independence Constraint Qualification}
\newacronym{sqp}{SQP}{Sequential Quadratic Programming}
\newacronym[plural=FLOPs, firstplural=Floating Point Operations (FLOPs)]{flop}{FLOP}{Floating Point Operation}
\newacronym{cpu}{CPU}{Central Processing Unit}
\newacronym{jit}{JIT}{Just-In-Time}
\newacronym{xla}{XLA}{Accelerated Linear Algebra}
\newacronym{rmse}{RMSE}{Root Mean Square Error}
\newacronym{rti}{RTI}{Real-Time-Iteration}
\newcommand{\equref}[1]{Equation~\eqref{#1}}

\newtheorem{assumption}{Assumption}
\newtheorem{theorem}{Theorem}
\begin{document}
\maketitle
\thispagestyle{empty}
\pagestyle{empty}

\begin{abstract}
The control of high-dimensional systems, such as soft robots, requires models that faithfully capture complex dynamics while remaining computationally tractable. This work presents a framework that integrates \gls{gnn}-based dynamics models with structure-exploiting Model Predictive Control to enable real-time control of high-dimensional systems. By representing the system as a graph with localized interactions, the \gls{gnn} preserves sparsity, while a tailored condensing algorithm eliminates state variables from the control problem, ensuring efficient computation. The complexity of our condensing algorithm scales linearly with the number of system nodes, and leverages \gls{gpu} parallelization to achieve real-time performance. The proposed approach is validated in simulation and experimentally on a physical soft robotic trunk. Results show that our method scales to systems with up to 1,000 nodes at 100~Hz in closed-loop, and demonstrates real-time reference tracking on hardware with sub-centimeter accuracy, outperforming baselines by 63.6\%. Finally, we show the capability of our method to achieve effective full-body obstacle avoidance.
\end{abstract}

\vspace{0.2cm}
\noindent {\small {\bf Website}: \href{https://gnn-mpc.github.io/}{https://gnn-mpc.github.io/}}
\section{Introduction}

Robotic systems with high \glspl{dof} offer exceptional versatility and adaptability. Continuum soft robots exemplify this potential whereby their compliant and elastic structures naturally conform to complex surfaces and objects. This intrinsic compliance not only enhances physical robustness but also makes soft robots particularly well-suited for operation in delicate, constrained environments. Applications include safe locomotion, manipulation, and medical procedures~\cite{Rus2015}.

However, the features that provide these advantages also pose significant challenges for modeling and control. Accurately capturing their nonlinear dynamics and deformations demands complex, high-dimensional models. Existing approaches attempt to address this by either employing physics-based models~\cite{webster_design_2010} or data-driven methods~\cite{bruder_modeling_2019, alora2025discovering}; yet, they face an inherent tradeoff between model fidelity and computational tractability. Moreover, most modeling approaches restrict control to a subset of the state space for computational tractability, thereby restricting the ability to impose constraints or optimize performance across an extended set of points along a robot's structure. Ultimately, the full potential of soft-robot versatility and adaptability hinges on the effectiveness of the underlying control methods.

% Concurrently, \glspl{gnn} have emerged as powerful models for representing complex system dynamics. Previous work has demonstrated that \glspl{gnn} can act as efficient, data-driven physics engines, achieving substantial speedups over traditional numerical simulations~\cite{battaglia_interaction_2016}. \glspl{gnn} define neural networks on graph elements and a convolution operator, allowing them to propagate information between nodes and edges, using the structural and relational patterns inherent in high-dimensional robotic platforms. In combination with \gls{mpc}~\cite{rawlings2020model}, which provides a powerful and widely adopted framework by formulating control as a finite-horizon constrained optimization problem, \glspl{gnn} have the potential to enable real-time optimal control of high-dimensional systems. However, integration of \glspl{gnn} into \gls{mpc} frameworks remains largely unexplored, and existing methods often struggle to scale to the dimensionality present in complex robotic systems. In particular, \glspl{gnn} preserve the full state dimension, which leads to high solver times for large-scale problems.

Concurrently, \acrfullpl{gnn} have emerged as powerful, data-driven models for representing complex dynamical systems, achieving substantial speedups over traditional numerical simulations~\cite{battaglia_interaction_2016}. \glspl{gnn} operate on graph elements and naturally encode relational patterns inherent in high-dimensional robotic platforms, mirroring their sparse physical structure. In combination with \gls{mpc}~\cite{rawlings2020model}, they have the potential to enable real-time optimal control at scale. However, integrating \glspl{gnn} into \gls{mpc} remains largely underexplored, and current formulations do not leverage their sparse structure, leading to prohibitive solve times for large-scale problems.

\subsubsection*{Contributions}
\begin{figure}[!t]
\centerline{\includegraphics[width=0.95\columnwidth]{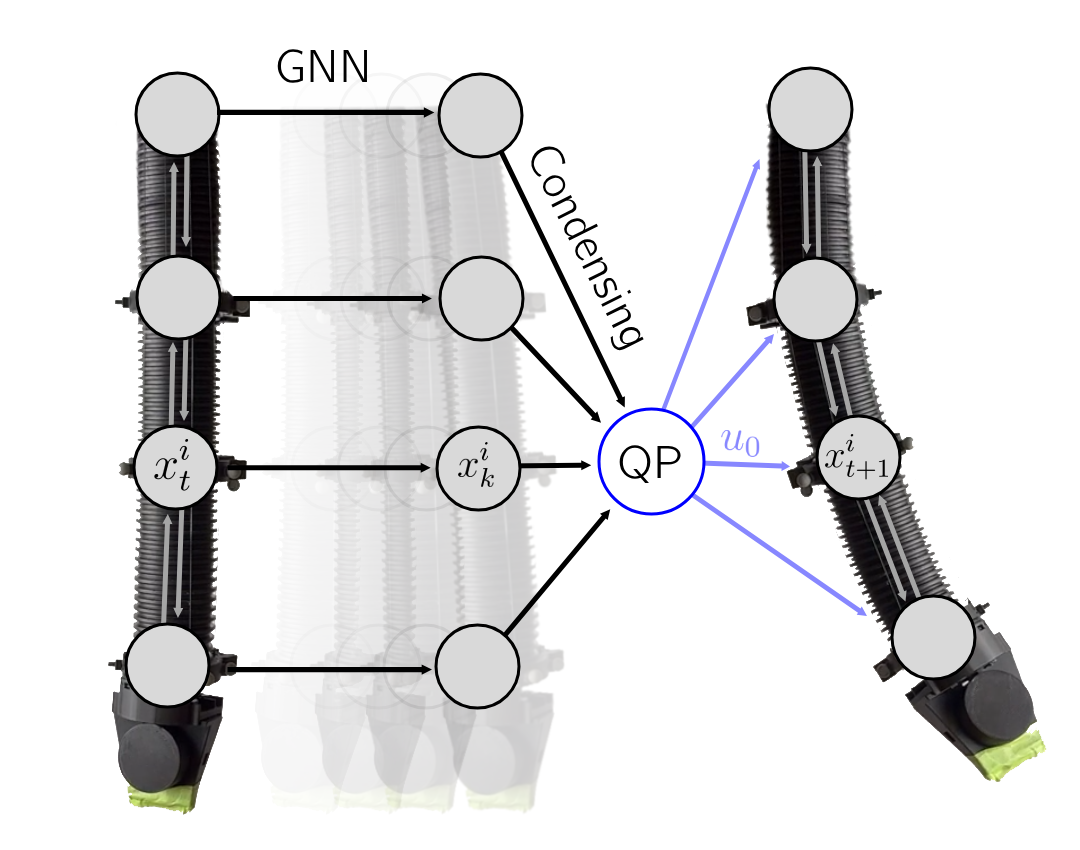}}
\caption{A soft robotic trunk is modeled as a graph with nodes as discrete segments with state $x^i_t$ at time $t$ and edges as physical interactions. A \gls{gnn} computes the linearized forward dynamics $x^i_k$, which are condensed into a Quadratic Program (QP) that only depends on the input variables $u_k$. The QP is then solved, and the first input $u_0$ is applied in a receding-horizon fashion.}
\label{fig:header}
\end{figure}

We propose a framework that, under the assumption of localized interactions, models dynamical systems as graphs and uses \glspl{gnn} to learn structured, scalable representations of their dynamics. To enable real-time optimal control, we adapt existing \textit{condensing} algorithms from~\cite{frison_algorithms} to eliminate the state variables from the resulting \gls{ocp}, thus reducing its dimensionality without compromising model fidelity. This is achieved by leveraging the localized sparsity structure of the linearized dynamics equations provided by the \gls{gnn}, leading to linear scaling with respect to the number of graph nodes in the modeled system, provided that each node interacts with a bounded number of neighbors. We summarize our list of contributions below:
\begin{enumerate}
    \item \textbf{\gls{gnn} dynamics with structure-aware condensing for \gls{mpc}:} We enable relational modeling of system dynamics and preserve structural sparsity, which we leverage to design a condensing algorithm with linear scaling properties in system size for a bounded number of neighbors. Figure~\ref{fig:header} provides an intuitive overview of the approach.
    \item \textbf{Highly efficient implementation:} We leverage parallel operations on the \gls{gpu} to achieve near-constant computation times for systems with up to 1,000 nodes at 100~Hz control frequency.
    \item \textbf{Experimental validation:} We demonstrate reduced trajectory tracking error on a physical soft robotic trunk by 63.6\%, and showcase direct control over all observable system states.
\end{enumerate}

\subsubsection*{Outline}
We first provide a literature review (Sec. \ref{sec:related_work}), followed by a theoretical background on \glspl{gnn} and condensing (Sec. \ref{sec:background}). We then provide a formal problem statement (Sec. \ref{sec:problem_formulation}) and detail the proposed \gls{gnn}-\gls{mpc} framework (Sec. \ref{sec:gnn_mpc}). Finally, we provide results from simulation and hardware experiments (Sec. \ref{sec:experiments}) and present our conclusions (Sec. \ref{sec:conclusion}).

\subsubsection*{Notation}
The set of real numbers and integers are represented by $\mathbb{R}$ and $\mathbb{Z}$, respectively. We denote the set of integers in the interval $[a,b]$ as $\mathbb{Z}_{[a,b]}$. For $Q \succeq 0$, we define the weighted squared seminorm $\|x\|_Q^2 \coloneqq x^\top Q x$.

\section{Related Work}
\label{sec:related_work}
Direct modeling of soft robots typically relies on continuum mechanics and multibody dynamics. While 3D solid mechanics and the \gls{fem} provide accurate representations~\cite{mengaldo_concise_2022}, solving the resulting differential equations is computationally expensive, and accurate characterization of material properties remains challenging. To reduce computational cost, reduced-order models such as the piecewise-constant curvature model~\cite{webster_design_2010,della_santina_dynamic_2018} approximate a soft robot through constant-curvature segments, at the expense of failing to capture more complex dynamics. In addition, Koopman operator theory enables the construction of linear models for nonlinear systems by lifting the dynamics into a higher-dimensional space~\cite{bruder_modeling_2019}. However, exact finite-dimensional linear representations of most physical systems are generally not available~\cite{brunton_koopman_2016}, and prediction performance depends on the quality of the finite-dimensional approximation. Model reduction can be further achieved through \glspl{ssm}~\cite{haller_nonlinear_2016}, which define low-dimensional attracting invariant manifolds that approximate the dominant nonlinear behaviors of high-dimensional systems. Recent work has shown that dynamics can be learned and controlled efficiently on these manifolds using data-driven approaches~\cite{alora2025discovering}. However, the use of \glspl{ssm} often imposes stringent assumptions, such as moderate control magnitudes and asymptotic stability of the system's origin. Instead, we propose a complementary method that relaxes these assumptions and represents high-dimensional systems as sparse networks of interacting elements.

Beyond structured physics-based models, neural networks provide an alternative modeling approach for dynamical systems. Feedforward networks can learn linearized state-space models from observed dynamics~\cite{gillespie_learning_2018}, facilitating real-time control while lessening the associated modeling effort. Nevertheless, such neural networks might experience difficulties in generalizing the system's dynamics due to the absence of inductive biases, i.e. structural assumptions and prior knowledge which are reflected in the model architecture and the learning process. In contrast, \glsentrylong{gnn}s incorporate relational structure through local message-passing mechanisms to capture interactions between system components~\cite{battaglia_interaction_2016,bronstein_geometric_2021}. Further developments have shown that \glspl{gnn} can outperform standard convolutional networks in mesh-based simulations, fluid dynamics, and forward dynamics prediction~\cite{pfaff_learning_2021}. \glspl{gnn} have also been applied to online planning, including unconstrained trajectory optimization via gradient descent~\cite{sanchez-gonzalez_graph_nodate,li_learning_2019}, and have also shown robustness under partial observability~\cite{li_propagation_2019}. Applications in robotics include soft robotic hands~\cite{almeida_sensorimotor_2021}, where the relational inductive bias of \glspl{gnn} enables reasoning over physical interactions more effectively than unstructured neural networks. \glspl{gnn} have also demonstrated their value in policy learning~\cite{gammelli_graph_2021, gammelli_graph_2022}, showing that \gls{gnn}-based policies exhibit greater transferability, generalization, and scalability in reinforcement learning tasks than non-relational approaches. Despite this growing body of work, the use of \glspl{gnn} for efficient closed-loop optimal control of high-dimensional systems remains largely unexplored. Our work addresses this gap by combining \glspl{gnn} with \gls{mpc} for real-time, high-dimensional control.

A major limitation in using expressive models such as \glspl{gnn} within \gls{mpc} lies in the computational complexity of solving the resulting \gls{ocp} using Newton-type methods, such as \gls{sqp}. Most sparsity-exploiting \gls{ipm} \gls{qp} solvers suffer from cubic scaling with respect to the state dimension~\cite{frison_hpipm_2020}. To address this, various structure-exploiting algorithms have been developed, among which condensing stands out as a particularly relevant approach for our problem~\cite{frison_algorithms}. Condensing removes state variables from the \gls{ocp} at the cost of an additional preprocessing step~\cite{frison_hpipm_2020,frison_efficient_2016,jerez_sparse_2012}. While it preserves problem structure and reduces solver times, its computational overhead scales unfavorably with system dimensionality. In this work, we leverage the inherent sparse structure of the \glspl{gnn} to develop a condensing algorithm which supports computationally tractable \gls{mpc} for high-dimensional systems.

%Note: Neton's method is second-order because the error decreases O(e^2) to previous iteration error. Not because it uses Hessian. IPM solvers are second-order since they apply Newton's method on the KKT system (which includes the derivative of the Lagrangian), not because they leverage any Hessian directly, but as a consequence of applying a derivative on the KKT system. Gradient descent is a first order method.

% ADMM converges sublinearly (O(1/k) or O(1/k^2), k is the iteration index) or linearly at best. IPM solvers converge quadratically
\section{Background}
\label{sec:background}

This section provides a brief overview of \glspl{gnn} in \ref{subsec:gnns} and condensing for \gls{mpc} in \ref{subsec:condensing}, which form the foundation of our proposed approach.

\subsection{Graph Neural Networks}
\label{subsec:gnns}
\glsentrylong{gnn} generalize neural architectures to graph-structured data. At their core, \glspl{gnn} exploit the relational structure in the data by reusing a set of local learnable functions, typically defined on nodes and edges, and applying them across the entire graph through convolution-like operators. This design enforces permutation equivariance and, using the symmetries inherent in graphs, \glspl{gnn} have achieved remarkable success across a wide range of disciplines~\cite{battaglia_interaction_2016, pfaff_learning_2021, sanchez-gonzalez_learning_2020}. To formalize this notion, we represent the network as a graph $\mathcal{G}=(\mathcal{V}, \mathcal{E})$, where $\mathcal{V}$ and $\mathcal{E}$ denote the set of nodes or vertices and directed edges, respectively, where each node $v_i \in \mathcal{V}$ encodes a local state $x^i \in \mathbb{R}^{\bar n_x}$ and each edge $(v_i,v_j) \in \mathcal{E}$ encodes an interaction $e^{ij} \in \mathbb{R}^{n_e}$ from node $v_j$ to $v_i$. Various \gls{gnn} architectures exist, typically classified into convolution, attention, and message-passing architectures~\cite{bronstein_geometric_2021}.

In the message-passing framework~\cite{battaglia_interaction_2016}, each node attribute $x^i_k$ at layer $k$ is updated by aggregating information from its neighbors as
\begin{align}
x^i_{k+1} = \phi \Big(x^i_k, \bigoplus_{j \in \mathcal{N}_i} \psi(x^i_k, x^j_k, e^{ij}_k) \Big),
\end{align}
where $\psi$ and $\phi$ are learnable functions, $\oplus$ is a permutation-invariant aggregation operator (e.g., sum or mean pooling), and $\mathcal{N}_i \coloneqq \{v_j \in \mathcal{V} \mid (v_j, v_i) \in \mathcal{E}\}$ denotes the neighborhood of node $i\in\mathcal{M}$, where $\mathcal{M} \coloneqq \mathbb{Z}_{[1,M]}$  is the set of node indices with $M = |\mathcal{V}|$ nodes. This formulation mirrors physical laws, learning local node functions that aggregate neighboring interactions. This motivates our approach for modeling high-dimensional dynamics with \glspl{gnn}, where forces are exchanged between system components. 

\subsection{Condensing in \gls{mpc}}
\label{subsec:condensing}
We consider a linear time-varying system with linear constraints and a quadratic cost function defined over a finite-horizon $N$\footnote{Section \ref{sec:gnn_mpc} describes how the nonlinear \gls{ocp} is solved with \gls{sqp} by a linear approximation.}, leading to the following \gls{ocp} formulation:
\begin{subequations}
\begin{align}
\min_{x_k, u_k} \quad 
& \sum_{k=0}^{N-1} \|x_k\|_{Q_k}^2 + \|u_k\|_{R_k}^2 + q_k^T x_k + r_k^T u_k \\
& + \|x_N\|_{Q_N}^2 + q_N^T x_N
\label{eq:ocp_cost_norm} \\
\text{s.t.} \quad & x_{k+1} = A_k x_k + B_k u_k + c_k, \quad k \in \mathbb{Z}_{[0, N-1]}, \label{eq:ocp_dynamics} \\
& x_k \in \mathbb{X}_k = \{ x \mid C_k^x x \le d_k^x \}, \quad k \in \mathbb{Z}_{[0, N]},\\
& u_k \in \mathbb{U}_k = \{ u \mid C_k^u u \le d_k^u \}, \quad k \in \mathbb{Z}_{[0, N-1]},\\
& x_0 = \tilde x(t), \label{eq:ocp_initial_state}
\end{align}
\end{subequations}
where $Q_k \succeq 0 \in \mathbb{R}^{\bar n_x \times \bar n_x}$, $R_k \succeq 0 \in \mathbb{R}^{n_u \times n_u}$ are cost matrices and $q_k \in \mathbb{R}^{\bar n_x}$, $r_k \in \mathbb{R}^{n_u}$ represent the linear terms of the stage cost. \equref{eq:ocp_dynamics} describes the linear system dynamics, where the initial condition $x_0$ is set to $\tilde x(t)$, the measured state of the system at a given time $t$. The state and input constraints $ \mathbb{X}_k,  \mathbb{U}_k$ are time-varying linear half-space constraints. We define the augmented variables
\[
x = [x_0^\top, \dots, x_N^\top]^\top, \quad
u = [u_0^\top, \dots, u_{N-1}^\top]^\top,
\]
and the corresponding block-diagonal matrices and vectors
\(\bar{Q}, \bar{q}, \bar C_x, \bar d_x \) (over \(k = 0,\dots,N\)), and
\(\bar{R}, \bar{r}, \bar C_u, \bar d_u \) (over \(k = 0,\dots,N-1\))
as the stacked sequences of the stage matrices and vectors. The original problem can be reformulated by \textit{condensing}, i.e. eliminating state variables through forward substitution~\cite{frison_algorithms, frison_efficient_2016}. This reduces the \gls{ocp} into a \gls{qp} in the input sequence alone, with $N n_u$ optimization variables, i.e.,
\begin{align}
\min_{u}\, u^\top H u + g^\top u, \quad \text{s.t. } \tilde C u \leq \tilde d,
\end{align}
with
\begin{subequations} \label{eq:condensed_matrices}
\begin{align} 
H &= \Gamma_u^\top \bar{Q} \Gamma_u + \bar{R}, \quad
g = \Gamma_u^\top \bar{Q} \Gamma_x + \Gamma_u^\top \bar{q} + \bar{r},\\
\tilde{C} &= \begin{bmatrix} \bar{C}_u \\ \bar{C}_x \Gamma_u \end{bmatrix}, \quad 
\tilde{d} = \begin{bmatrix} \bar{d}_u \\ \bar{d}_x - \bar{C}_x \Gamma_x \end{bmatrix}.
\end{align}
\end{subequations}
The matrices $\Gamma_u \in \mathbb{R}^{N n_u \times N n_x}$ and $\Gamma_{x} \in \mathbb{R}^{N n_x}$ depend on the system dynamics matrices $A_k$, $B_k$, $c_k$, and are computed recursively, satisfying $x = \Gamma_u u + \Gamma_x$\footnote{See Section~\ref{subsec:condensing_for_high_dim_sys} for a detailed derivation.}. Condensing reduces the number of decision variables to be solved for, yet this benefit comes at the expense of at least $\mathcal{O}(n_x^2)$ additional floating point operations per solver iteration~\cite{frison_algorithms}.

\section{Problem Formulation}
\label{sec:problem_formulation}
We now introduce our assumptions on the system dynamics and formally define the control task.
\subsection{Local System Dynamics}
We consider a discrete-time nonlinear system which can be described as a composition of $M$ subsystems, where each subsystem $i \in \mathcal{M}$ has a local state ${x}^i \in \mathbb{R}^{\bar{n}_x}$ and receives a global external input $u_k \in \mathbb{R}^{n_u}$. The full system state is given by $x_t = [({x}^1_t)^\top, \dots, ({x}^M_t)^\top]^\top \in \mathbb{R}^{n_x}$ with $n_x = M \bar n_x$. Motivated by the typically uniform material properties throughout the continuum of a soft robot, we make the following assumption:
\begin{assumption} \label{ass:local_dynamics}
The forward dynamics of the overall state ${x}_{t+1} = {\mathcal{F}}({x}_t, {u}_t)$ can be expressed as
\begin{align} \label{eq:local_dynamics}
    {x}^i_{t+1} = f(x^i_t, \{{x}^j_t \}_{j \in \mathcal{N}_i}, u_t), \quad i \in \mathcal{M},
\end{align}  
where $\mathcal{N}_i \subseteq \mathcal{M}$ is the set of subsystems affecting subsystem~$i$, and ${f}: \mathbb{R}^{\bar{n}_x} \times \mathbb{R}^{|\mathcal{N}_i| \bar{n}_x} \times \mathbb{R}^{n_u} \to \mathbb{R}^{\bar{n}_x} $ 
is homogeneous across all subsystems and almost everywhere differentiable.
\end{assumption}
We further assume that each subsystem interacts with a small and bounded neighborhood, based on the local behavior of each material
point which depends on the immediate surroundings it is in contact with, i.e.,
\begin{assumption} \label{ass:small_neighborhood}
There exists a small bound $d$ such that
\begin{align}
   |\mathcal{N}_i| \leq d, \quad i \in \mathcal{M}.
\end{align}
\end{assumption}

Throughout, we consider systems where the overall state dimension of the system is much larger than the number of control inputs, and we consider that each subsystem interacts with a small neighborhood, i.e.,
\begin{align}
    n_u \ll n_x, \quad d \ll M.
\end{align}
%Note that continuum soft robots satisfy the previous assumptions, where we consider homogeneity such that a material's properties are identical at all points within the continuum and consider locality, where a material point's behavior is determined solely by its immediate surroundings. Further the number of actuators is typically much smaller than the number of states required to describe a soft robot's configuration.
Continuum soft robots naturally satisfy the assumptions outlined above. Moreover, in typical soft robotic systems, the number of actuators is significantly smaller than the number of states required to fully describe the robot’s configuration, motivating the use of condensing to solve a \gls{qp} in the input variables only.

\subsection{Control Task}
Our objective is the tracking of reference trajectories $\{x^{r}_t, u^{r}_t\}_{t\geq 0}$ across all subsystems $i \in \mathcal{M}$, subject to the state and input constraints that govern the overall system. At each time $t$, we solve the following \gls{ocp} with horizon $N$:
\begin{subequations} \label{eq:linear_gnn_mpc_ocp_explicit}
\begin{align}
    \min_{x_k, u_k} \quad & \sum_{i=1}^{M} \left (\sum_{k=0}^{N} \left( \|x^i_k - x^{i,r}_k\|^2_{Q_k^i} \right) \right )  + \sum_{k=0}^{N-1} \|u_k - u^r_k\|^2_{R_k}\\
    \text{s.t.} \quad & x^i_0 = \tilde x^i(t), \\
    %& x^i_k = A^i_k x^i_k + \sum_{j \in \mathcal{N}_i} A^{ij}_k x^j_k + B^i_k u_k + c^i_k,\\
    & x^i_{k+1} = f(x^i_k, \{x^j_k\}_{j \in \mathcal{N}_i}, u_k)\\
    & x^i_k \in \mathbb{X}^i_k, \quad x^i_N \in \mathbb{X}^i_N, \quad u_k \in \mathbb{U}_k, \\
    & k \in \mathbb{Z}_{[0, N-1]}, \quad i \in \mathcal{M},
\end{align}
\end{subequations}
where $\tilde x^i(t)$ is the measured state of subsystem $i$ at time $t$, $Q^i_k \in \mathbb{R}^{\bar n_x \times \bar n_x}, \, Q^i_k \succeq 0$ is the cost matrix applied to subsystem $i$ at stage $k$, $R_k \in \mathbb{R}^{n_u \times n_u}, \, R_k \succeq 0$ is the cost applied to the input at stage $k$,  $\mathbb{X}^i_k = \left\{ x \in \mathbb{R}^{\bar n_x} \mid C^{x,i}_k x \leq d^{x,i}_k \right\}$ is a convex set of half-space constraints defined for subsystem $i$ at time step $k$, and $\mathbb{U}_k = \left\{ u \in \mathbb{R}^{n_u} \mid C_k^u u \leq d^u_k \right\}$ defines the admissible control inputs at time step $k$.
\section{Graph Neural Model Predictive Control}
\label{sec:gnn_mpc}

We now introduce our proposed method for the tractable control of high-dimensional systems through \glspl{gnn} and condensing. We first detail the modeling of the system dynamics, and then derive a condensing method which yields linear scaling in the number of subsystems $M$ \footnote{Assuming a fixed horizon $N$, bounded neighborhood $d$, constant state dimension per node $\bar n_x$, and input dimension $n_u$.}.

\subsection{Modeling Dynamics with GNNs}

\subsubsection*{Model}
We represent the forward dynamics of the system as a second-order model, where each subsystem is characterized by its position and velocity, i.e. $x^i_k = [(p^i_k)^\top, (v^i_k)^\top]^\top$. We employ a \gls{gnn} based on the interaction network~\cite{battaglia_interaction_2016}, where nodes correspond to subsystems, edges encode their pairwise interactions, and we model a global input vector representing the control inputs that influence all subsystems simultaneously. We let the \gls{gnn} predict a velocity increment, which is integrated via a backward Euler scheme, such that the state $x^i_{k+1}$ is computed for each subsystem $i$ as
\begin{subequations}
\begin{align} \label{eq:forward_dynamics}
    p^i_{k+1} &= p^i_k + v^i_{k+1} \Delta t,\\
    v^i_{k+1} &= v^i_k + \phi\left(x^i_k, \sum_{j \in \mathcal{N}_i} \psi(e^{ij}_k), u_k\right),
\end{align}
\end{subequations}
where $\Delta t$ is the sampling period, $v$ denotes velocity in this context, and $\psi$ and $\phi$ are the message-passing and node-update functions, respectively. The edge feature $e^{ij}_k$ is computed as the relative state difference $e^{ij}_k = x^i_k - x^j_k$, enabling the network to capture interactions in local coordinates. The functions $\psi$ and $\phi$ represent nonlinear transformations that are implemented as \glspl{mlp} with \gls{relu} activations. Further, neighborhood aggregation is performed via summation, motivated by how forces accumulate in physical systems.

\subsubsection*{Dataset Generation}

The \gls{gnn} learns the dynamics from open-loop trajectories. We construct a dataset of $N_d$ elements composed of the measured states and applied control inputs at each time step, and the corresponding subsequent state:
\begin{align}
    \mathcal{D} = \{(x_t, u_t, x_{t+1})\}_{t=1,\dots,N_d}.
\end{align}
Measurements are collected at equidistant intervals $\Delta t$, equal to the intended sampling period of the \gls{mpc} controller.

\subsubsection*{Training}
Consequently, the dynamics are learned by minimizing the mean-squared-error loss function with $\ell_2$ regularization:
\begin{align}
    \mathcal{L}_{\theta}(\mathcal{D}) 
    \coloneqq \frac{1}{N_d} \sum_{t=1}^{N_d}
    \big\|x_{t+1} - \mathcal{F}_{\theta}(x_t, u_t) \big\|^2_O 
    + \lambda \|\theta\|^2,
\end{align}
where $\theta = (\theta_\phi, \theta_\psi)$ are the network parameters, $\mathcal{F}_{\theta}$ denotes the \gls{gnn} operating on the full graph $\mathcal{G}$, $O \succeq 0$ is a weighting matrix, and $\lambda$ is a regularization coefficient. We conduct a grid search over the hyperparameters to optimize predictive performance and train the models using the Adam optimizer~\cite{adam} with a learning rate scheduler based on patience.

\subsubsection*{Linearization}
To solve the resulting nonlinear \gls{mpc} problem with a Newton-type method, we require a linear approximation of the \gls{gnn} dynamics around a nominal trajectory $\{\hat{x}_k, \hat{u}_k\}$, which we obtain via automatic differentiation of each subsystem's local dynamics independently\footnote{At the non-differentiable point of the \gls{relu} activation ($x=0$), a subgradient of zero is used.}. The linearized dynamics of each subsystem $i$ at stage $k$ are given by
\begin{align} \label{eq:linearized_dynamics}
    x^i_{k+1} = A^{ii}_k x^i_k + \sum_{j \in \mathcal{N}_i} A^{ij}_k x^j_k + B^i_k u_k + c^i_k,
\end{align}
with $A^{ij}_k = \frac{\partial f^i_k}{\partial x^j_k}$,  $B^{i}_k = \frac{\partial f^i_k}{\partial u_k}$, and $f^i_k = f(x^i_k, \{x^j_k\}_{j \in \mathcal{N}_i}, u_k)$. The term $c^i_k$ is the resulting constant offset from linearization.

\subsection{Local-Scalable Condensing}
\label{subsec:condensing_for_high_dim_sys}
In this section, we present a tailored condensing algorithm for high-dimensional systems. Given the small neighborhood size $d$, the linearized dynamics of the overall system has a sparse structure, which is enforced by the \gls{gnn}. Therefore, we derive a condensing algorithm based on~\cite{frison_algorithms}\footnote{Various condensing algorithms exist in the literature. We adopt one that admits a local, per-subsystem decomposition.} that takes advantage of the local dynamics and their resulting sparsity, resulting in favorable scaling properties formalized in the following theorem.

\begin{theorem} \label{thm:linear_scaling}
    Let Assumptions~\ref{ass:local_dynamics} and \ref{ass:small_neighborhood} hold and consider the \gls{ocp} in \eqref{eq:linear_gnn_mpc_ocp_explicit}. Then, the condensing procedure in Section~\ref{subsec:condensing} scales linearly in the number of subsystems $M$, i.e. the computation of $\Gamma_u, \Gamma_x, H, g, \tilde C, \tilde d$ requires $\mathcal{O}(M)$ floating point operations and memory for fixed values of $\bar n_x, n_u, N, d$.
\end{theorem}
\begin{proof}
Consider the system dynamics in \equref{eq:linearized_dynamics} and a horizon $N=1$. We can equivalently describe the state evolution of subsystem $i$ as
\begin{align}
    x^i = \begin{bmatrix}
        x_0^i \\
       x^i_1\\
    \end{bmatrix} = \Gamma_u^i u_0 + \Gamma_x^i,
\end{align}
with
\begin{align}
    \Gamma_u^i = \begin{bmatrix}
        0 \\
        B^i_0\\
    \end{bmatrix}, \, 
    \Gamma_x^i = \begin{bmatrix}
        x_0^i \\
        \sum_{j \in \tilde{\mathcal N_i}} A^{ij}_{0} x^j_0 + c_0\\
    \end{bmatrix},
\end{align}
where the matrices $\Gamma_u^i$ and $\Gamma_x^i$ only depend on the local neighborhood $\tilde{\mathcal{N}}_i \coloneqq \mathcal{N}_i \cup \{i\}$. For larger horizons, each state $x^i_{k+1}$ is obtained recursively from the previous state $x^i_{k}$ based on \eqref{eq:linearized_dynamics}, and therefore we can compute $\Gamma_u^i$ and $\Gamma_x^i$ recursively as
\begin{align}
    \label{eq:sparse_gamma_i_xb}
    (\Gamma_{x}^i)_{n+1} &= \begin{bmatrix} (\Gamma_{x}^i)_{n} \\ \sum_{j \in \tilde{\mathcal N_i}} A^{ij}_{n}\mathcal{E}_{n}(\Gamma_{x}^j)_{n} + c^i_{n} \end{bmatrix},
\end{align}
where $(\Gamma_{x}^i)_{n} \in \mathbb{R}^{\bar n_x (n + 1)}$. Further, $\mathcal{E}_n = [0_{\bar n_x \times \bar n_x n} \; I_{ \bar n_x}]$ selects the last row of the matrix it operates on, and we denote the recursion index by the subscript $n = 1,\dots, N-1$. Similarly,
\begin{align}
    \label{eq:sparse_gamma_i_u}
    (\Gamma_{u}^i)_{n+1} &= 
    \begin{bmatrix} 
        (\Gamma_{u}^i)_{n} & \\ 
        \sum_{j \in \tilde{\mathcal N_i}} A^{ij}_{n} \mathcal{E}_{n} (\Gamma_u^j)_{n} & B^i_{n} \end{bmatrix},
\end{align}
where $(\Gamma_{u}^i)_{n} \in \mathbb{R}^{\bar n_x (n+1) \times n_u n}$, and the recursion terminates with $\Gamma_x^i = (\Gamma_x^i)_N, \, \Gamma_u^i = (\Gamma_u^i)_N$. Under Assumption~\ref{ass:small_neighborhood}, the computation of $\Gamma_{x}^i$ and $\Gamma_{u}^i$ depends only on $d$ neighboring nodes and is therefore independent of the overall system size $M$, which enables a distributed evaluation within each iteration.

Moreover, the cost function in \eqref{eq:linear_gnn_mpc_ocp_explicit} features separability for each subsystem $i$, i.e. the individual cost terms are aggregated, and we can compute the Hessian $H$ independently for each subsystem $i$ as
\begin{align} \label{eq:H}
    H = \sum_{i = 1}^{M} H^i + \bar{R}, \quad H^i = (\Gamma_u^i)^\top \bar Q^i \Gamma_u^i,
\end{align}
where the local Hessians $H^i$ are added to construct the full Hessian. Similarly, the gradient $g$ can be decomposed as
\begin{align} \label{eq:g}
    g = \sum_{i=1}^M g^i + \bar{r}, \quad g^i = \left( \Gamma_u^i \right)^\top \left(\bar{Q}^i \Gamma_{x}^i + \bar{q}^i \right),
\end{align}
where each subsystem contributes $g^i$. Therefore, the condensing procedure can be performed with $M$ distributed operations, where each subsystem computes its local contribution to the condensed \gls{qp}, which is aggregated in a final step. Finally, following the separability of the constraints in \eqref{eq:linear_gnn_mpc_ocp_explicit}, the latter can be processed independently and concatenated:
\begin{align}
\tilde{C} &= \begin{bmatrix} \bar{C}_u \\ \bar{C}_x^1 \Gamma_u^1 \\ \vdots \\ \bar{C}_x^M \Gamma_u^M \end{bmatrix}, \quad 
\tilde{d} = \begin{bmatrix} \bar{d}_u \\ \bar{d}_x^1 - \bar{C}_x^1 \Gamma_x^1 \\ \vdots \\ \bar{d}_x^M - \bar{C}_x^M \Gamma_x^M \end{bmatrix},
\end{align}
and we denote the corresponding entries of subsystem $i$ as $\tilde{C}^i, \tilde{d}^i$.
Therefore, the overall computational complexity grows linearly with the number of nodes in the system, i.e., it requires $\mathcal{O}(M)$ floating point operations for fixed values of $\bar n_x, n_u, N, d$. Since each subsystem stores only fixed-size quantities, the total memory requirement also scales $\mathcal{O}(M)$.
\end{proof}

The previous proof provides an efficient and distributed computation scheme to condense the quantities in \eqref{eq:condensed_matrices} by leveraging the sparse local graph dynamics. In addition, the computations are naturally suited for \gls{gpu} parallelization, since they are carried out for each node and combined efficiently through summation or stacking. This can significantly reduce solve times, as will be demonstrated in Section~\ref{sec:experiments}.

\subsection{\gls{gnn}-\gls{mpc} Algorithm}
%TODO SQP and linearized dynamics
The general procedure is summarized in Algorithm~\ref{alg:gnn-mpc}. At each time step, the dynamics are linearized around $\{\hat{x}_k, \hat{u}_k\}$ as part of an \gls{sqp} scheme, and the condensed \gls{qp} is constructed by aggregating the contributions from each subsystem. Consequently, it is then solved using an \gls{ipm} \gls{qp} solver, and the first control input is applied to the system. The linearization trajectory is updated with the shifted optimal solution and the process is repeated.

\begin{algorithm}[H]
\caption{\gls{gnn}-\gls{mpc}}
\label{alg:gnn-mpc}
\SetAlgoLined
\KwIn{$\{x_k^r, u_k^r\}$, $Q_k^i, R_k$, $N$, $\mathbb{X}^i_k, \mathbb{U}_k$, GNN}
\textbf{Initialize:} $\{\hat{x}_k, \hat{u}_k\} \gets \{x(t), 0\}$\; \{\textit{Set lin. trajectory}\}

\For{$t = 0, \Delta t, \dots$}{
    Set $x_0 \gets \tilde x(t)$\\
    \For{$k = 0,\dots, N-1$}{
        \For{$i = 1,\dots, M$}{
            Compute $A^{ij}_k,B^i_k,c^i_k$ around $\{\hat{x}_k, \hat{u}_k\}$\\ 
            for $j \in \mathcal{\tilde N}_i$\;\{\textit{Linearize GNN dynamics}\}
        }
    }
    \For{$i = 1,\dots, M$}{
        Compute $\Gamma^i_u, \Gamma^i_x, H^i, g^i, \tilde C^i, \tilde d^i$ \; \{\textit{Condense}\}
    }
    Compute $H, g, \tilde C, \tilde d$ \; \{\textit{Aggregate}\}

    $\min_{u} u^\top H u + g^\top u, \quad \text{s.t. } \tilde C u \leq \tilde d$ \; \{\textit{Solve \gls{qp}}\}

    $\{\hat{x}_k, \hat{u}_k\} \gets \{x_{k+1}^\star, u_{k+1}^\star\}$\; \{\textit{Update lin. trajectory}\}

    Apply $u_0^\star$ to system
}
\end{algorithm}
\section{Experimental Results}
\label{sec:experiments}
This section details the experimental validation of the proposed \gls{gnn}-\gls{mpc} algorithm. Following a description of the experimental setup, we present simulation and hardware results designed to quantify the model's open-loop prediction accuracy, assess its performance in closed-loop control, and demonstrate scalability, enabling new applications for high-dimensional systems.

\subsection{Setup}
We trained the message-passing \gls{gnn} in \textit{PyTorch Geometric}, and translated it to \textit{Flax} to enable \gls{jit} compilation and efficient batching. The condensing algorithm was implemented in \textit{JAX}, leveraging vectorized mapping and \textit{einsum}-based primitives for efficient \gls{gpu} execution via \gls{xla}. The resulting \gls{qp} is solved using the \gls{ipm} solver \textit{HPIPM}~\cite{frison_hpipm_2020} within a \gls{sqp} \gls{rti} scheme~\cite{sqp-rti}.

The experiments were conducted on an AMD Ryzen 9 7950X3D \gls{cpu} and an NVIDIA RTX 4090 \gls{gpu}. The simulation models a tendon-driven trunk robot in MuJoCo. Hardware experiments use the corresponding physical platform, tracked via an \textit{OptiTrack} motion capture system at 100~Hz and actuated by 6 XM540 Dynamixel motors. The software is implemented in ROS2.

\subsection{Simulation experiments}
This section first compares open-loop predictions of our approach with four baseline models, and then demonstrates the scalability of our method.
\subsubsection*{Open-loop Predictions}
We first evaluate the open-loop prediction accuracy of the learned \gls{gnn} dynamics model against a Koopman operator model~\cite{bruder_modeling_2019}, \gls{ssm} reduction with orthogonal projections~\cite{alora2025discovering}, \gls{ssm} with optimal linear projections~\cite{buurmeijer_taming_2025}, and a standard \gls{mlp} baseline. All models are trained on the same dataset distribution from the trunk simulator, consisting of quasi-periodic input trajectories with varying amplitudes and periods. Table~\ref{tab:benchmark-comparison} summarizes the average end effector \gls{rmse} on a one second horizon across 100 different trajectories. The \gls{gnn} achieves open-loop \gls{rmse} comparable to the Koopman method while surpassing all other baselines. As observed in \cite{alora2025discovering}, the Koopman operator performs better on certain simulation trajectories, supporting our results.

\begin{table}[h]
\caption{Average end-effector open-loop \gls{rmse} across models, reported as mean $\pm 1\sigma$ standard deviation.}
\centering
\begin{tabular}{l c}
\toprule
\textbf{Model} & \textbf{RMSE [m]} \\
\midrule
Koopman    & $(1.16 \pm 0.61) \times 10^{-3}$ \\
GNN        & $(1.38 \pm 0.76) \times 10^{-3}$ \\
SSM (Opt.)    & $(2.88 \pm 1.25) \times 10^{-3}$ \\
MLP        & $(3.59 \pm 1.67) \times 10^{-3}$ \\
SSM (Orth.)   & $(3.67 \pm 2.22) \times 10^{-3}$ \\
\bottomrule
\end{tabular}
\label{tab:benchmark-comparison}
\end{table}

\subsubsection*{Scalability}
Further, we evaluate the scalability of the proposed \gls{gnn}-\gls{mpc} framework by varying the number of segments in the simulated trunk robot from $1$ to $1000$. The \gls{mpc} problem is solved over a horizon of $N=20$ with input constraints and state constraints on the end effector, constituting a fixed subset of nodes as in Theorem~\ref{thm:linear_scaling}. Figure~\ref{fig:scalability} shows the average solve time over 100 runs for each robot configuration. The \gls{gnn}-\gls{mpc} exhibits sub-linear scaling with respect to the number of segments considered, enabling real-time control at 100~Hz.

\begin{figure}[h!]
\vspace{2mm}
\centerline{\includegraphics[width=0.95\columnwidth]{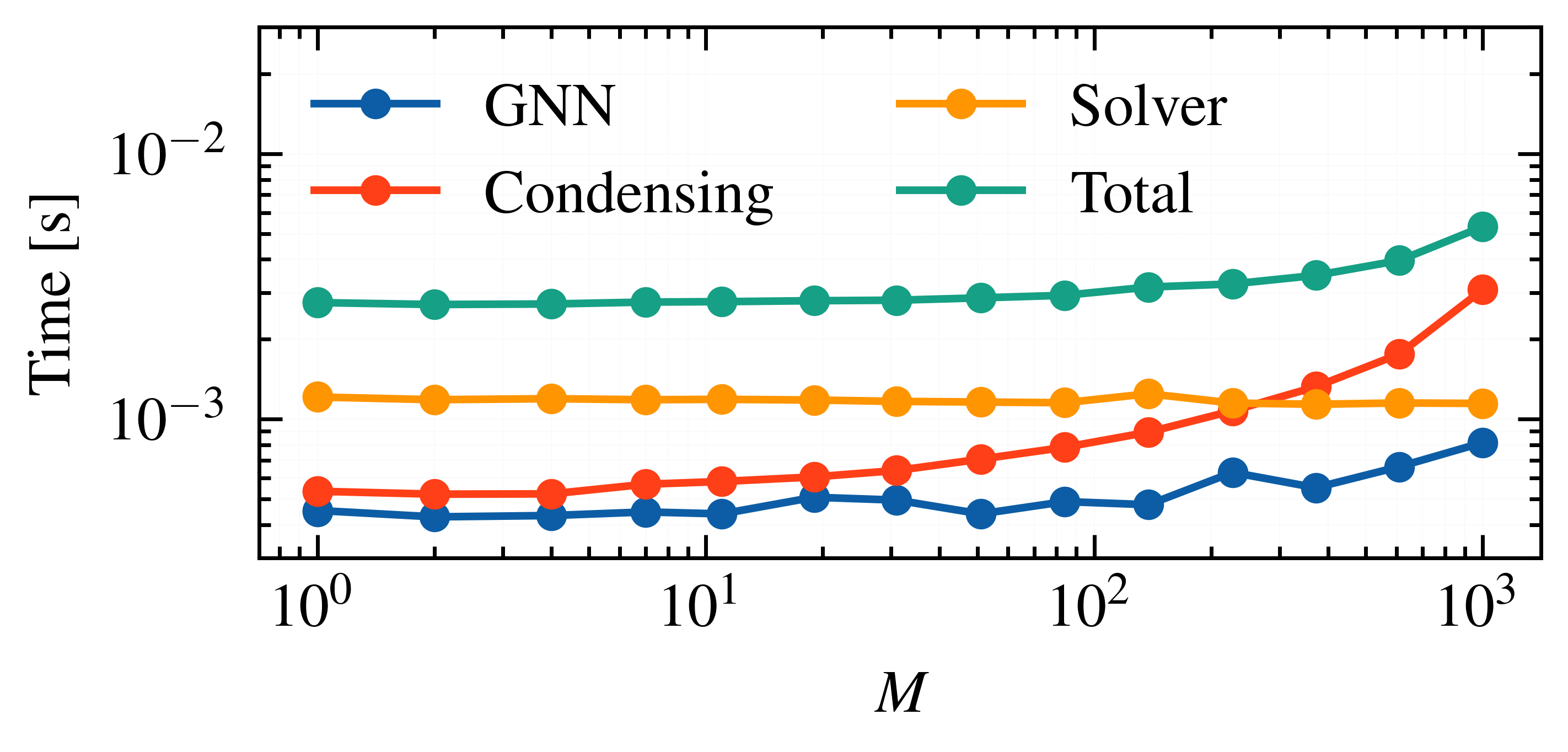}}
\caption{Scalability of the proposed \gls{gnn}-\gls{mpc} framework with respect to the number of subsystems $M$. The computation times of the relevant stages in Algorithm~\ref{alg:gnn-mpc} are shown in log-log scale.}
\label{fig:scalability}
\end{figure}

\begin{figure*}[!t]
    \vspace{2mm}
    \centering
    % First row
    \begin{subfigure}{0.25\textwidth}
        \centering
        \includegraphics[width=\linewidth]{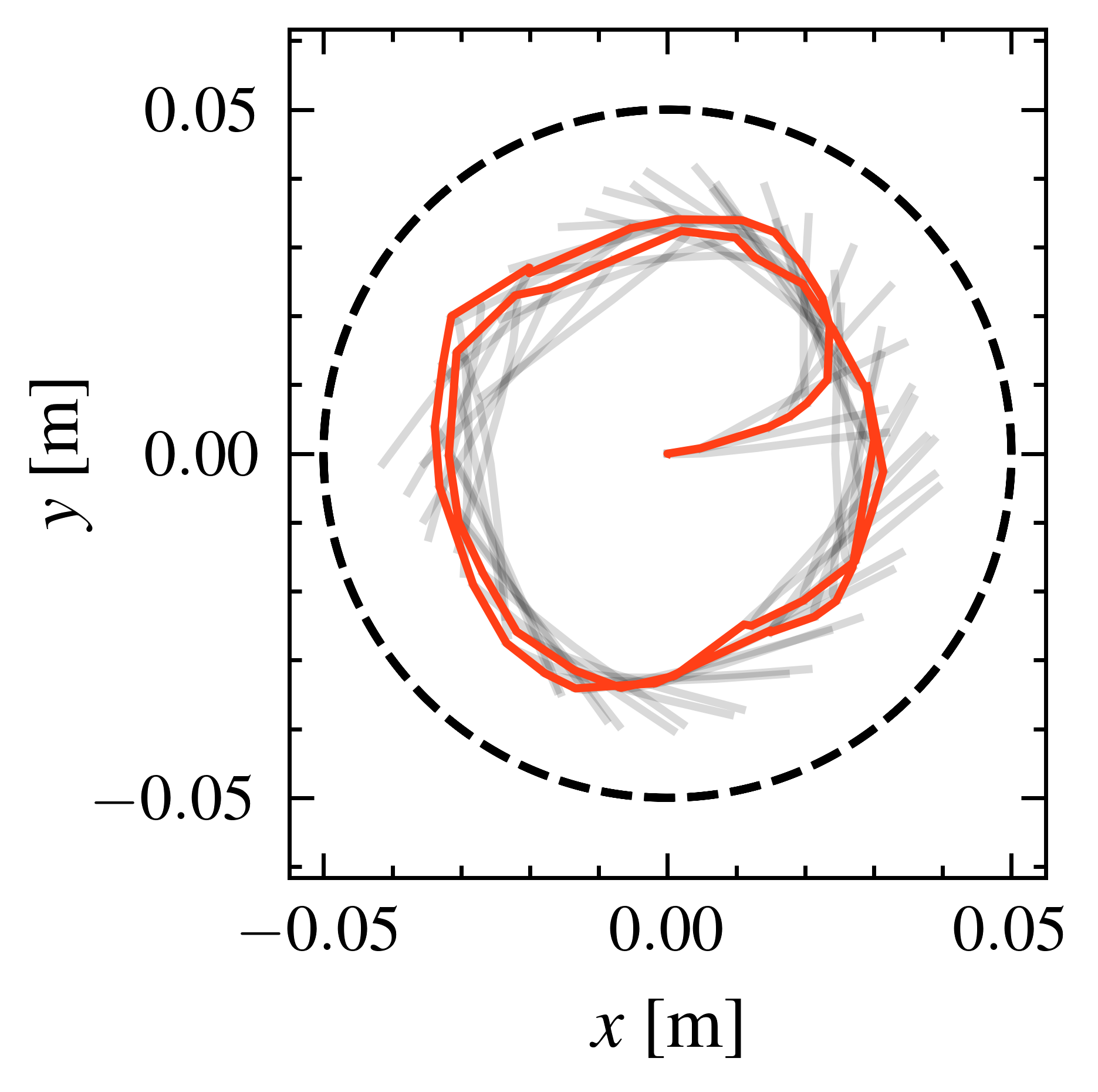}
        \label{fig:koopman-circle}
    \end{subfigure}
    \hspace{0.75cm}
    \begin{subfigure}{0.25\textwidth}
        \centering
        \includegraphics[width=\linewidth]{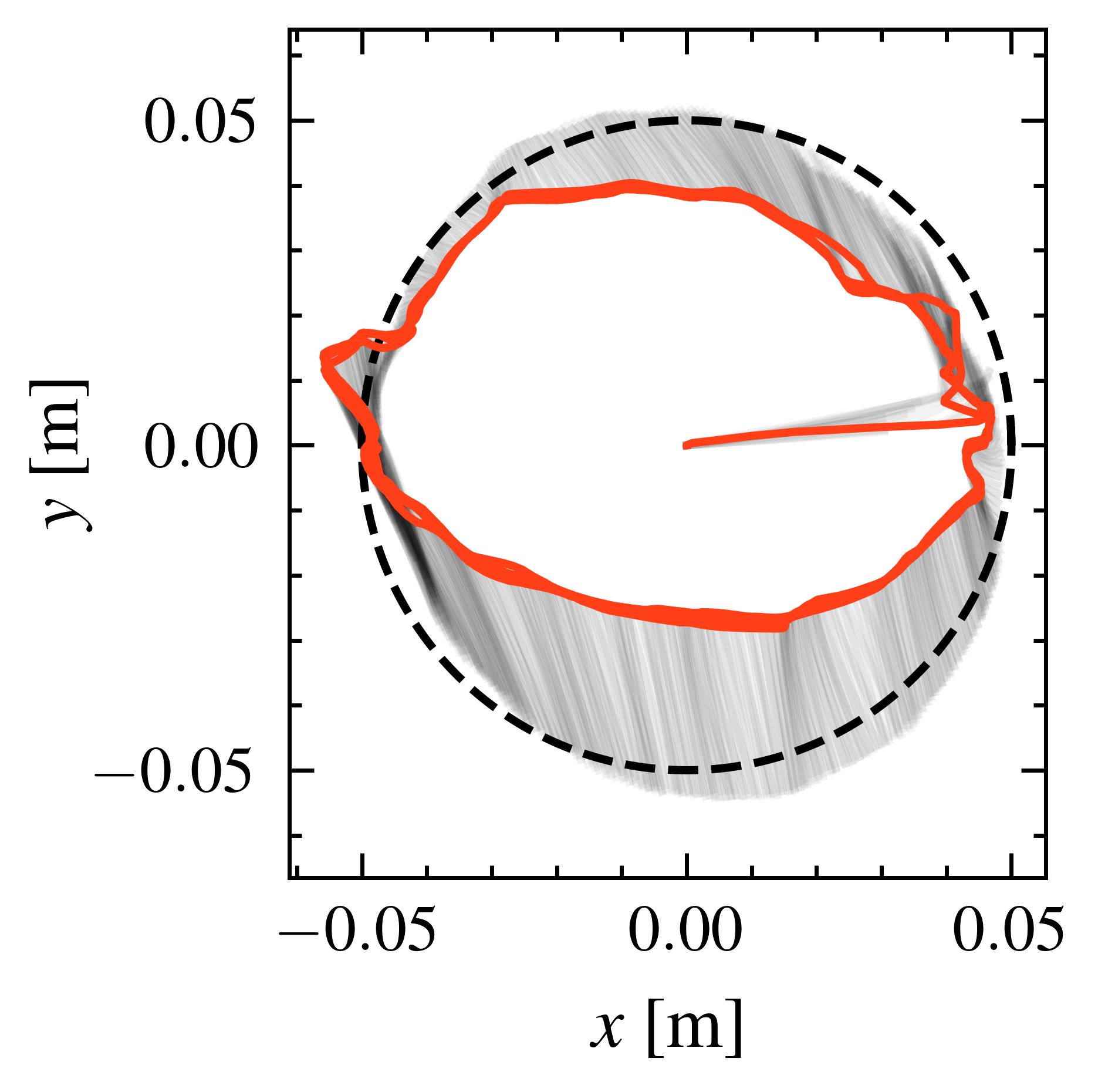}
        \label{fig:ssm-circle}
    \end{subfigure}
    \hspace{0.75cm}
    \begin{subfigure}{0.25\textwidth}
        \centering
        \includegraphics[width=\linewidth]{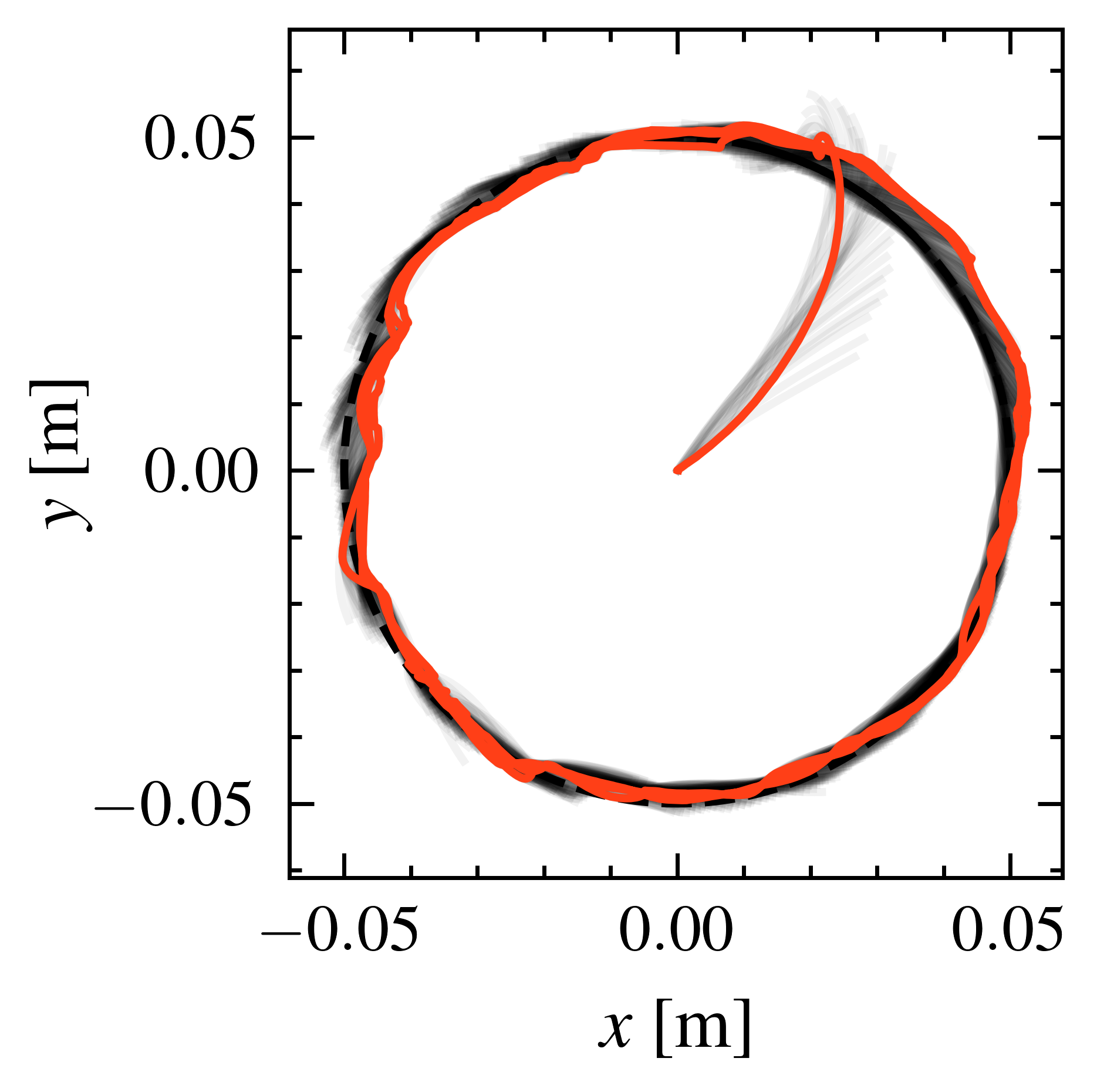}
        \label{fig:gnn-circle}
    \end{subfigure}

    \vspace{0.0em} % vertical space between rows

    % Second row
    \begin{subfigure}{0.25\textwidth}
        \centering
        \includegraphics[width=\linewidth]{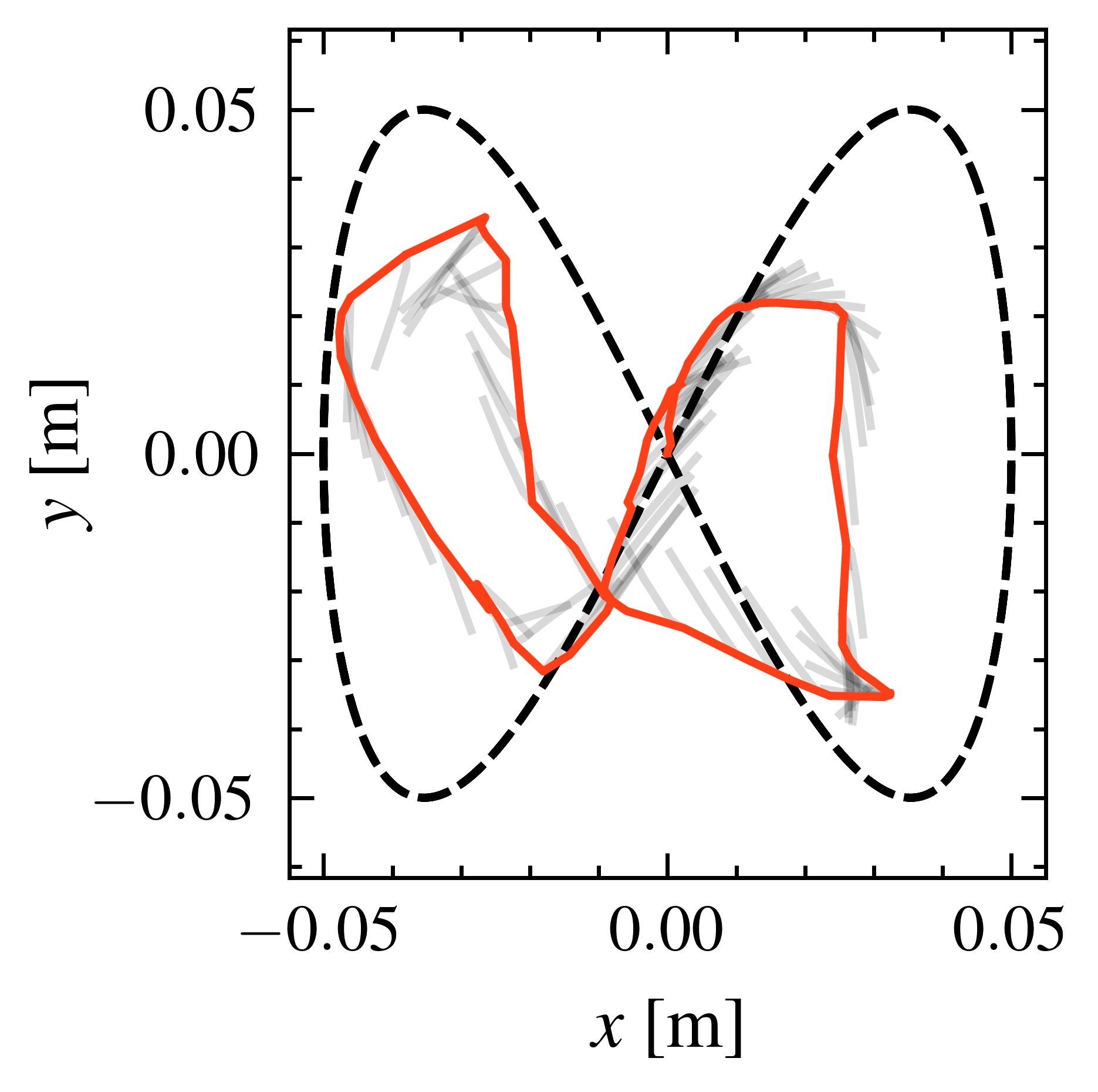}
        \caption{Koopman}
        \label{fig:koopman-eight}
    \end{subfigure}
    \hspace{0.75cm}
    \begin{subfigure}{0.25\textwidth}
        \centering
        \includegraphics[width=\linewidth]{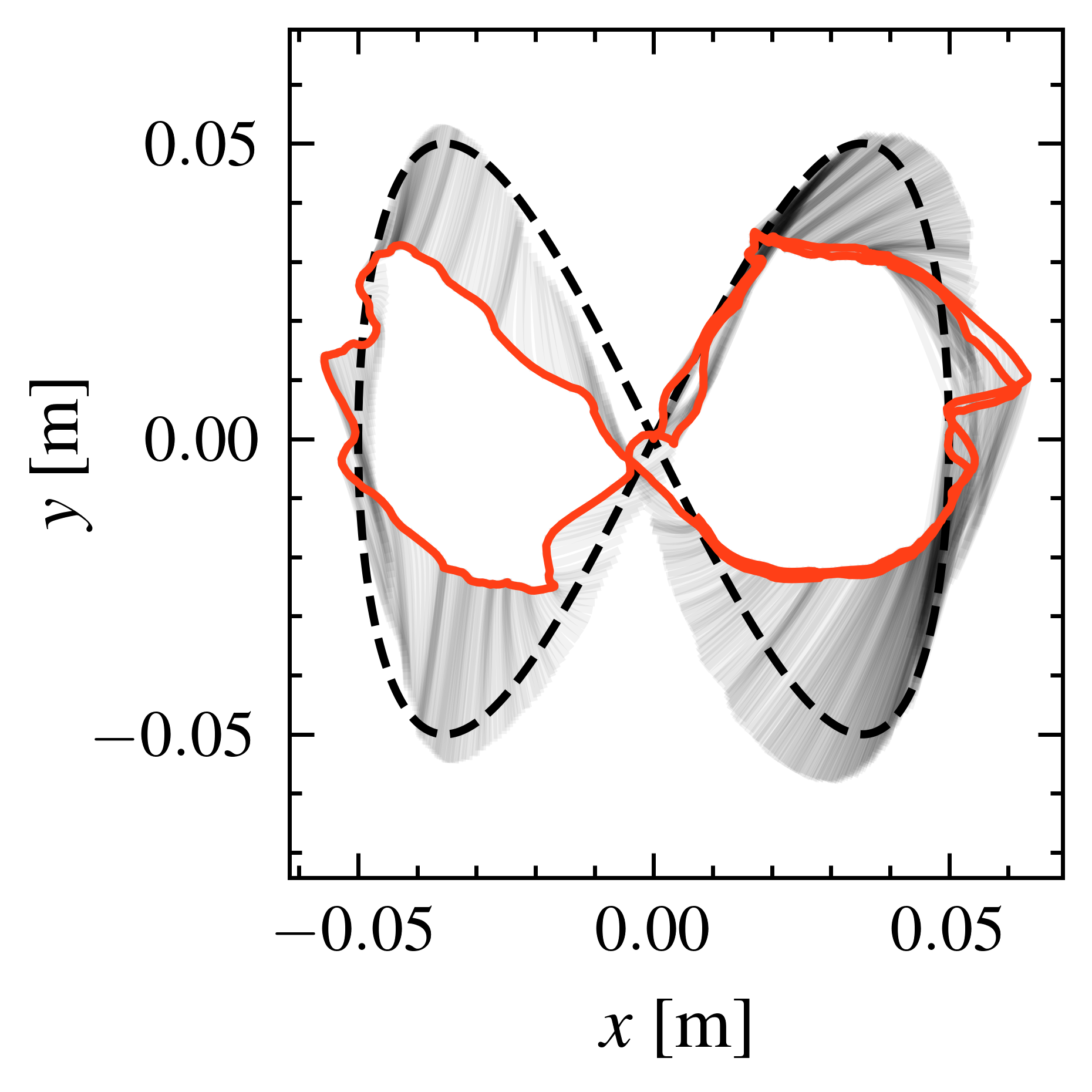}
        \caption{SSM-Orth}
        \label{fig:ssm-eight}
    \end{subfigure}
    \hspace{0.75cm}
    \begin{subfigure}{0.25\textwidth}
        \centering
        \includegraphics[width=\linewidth]{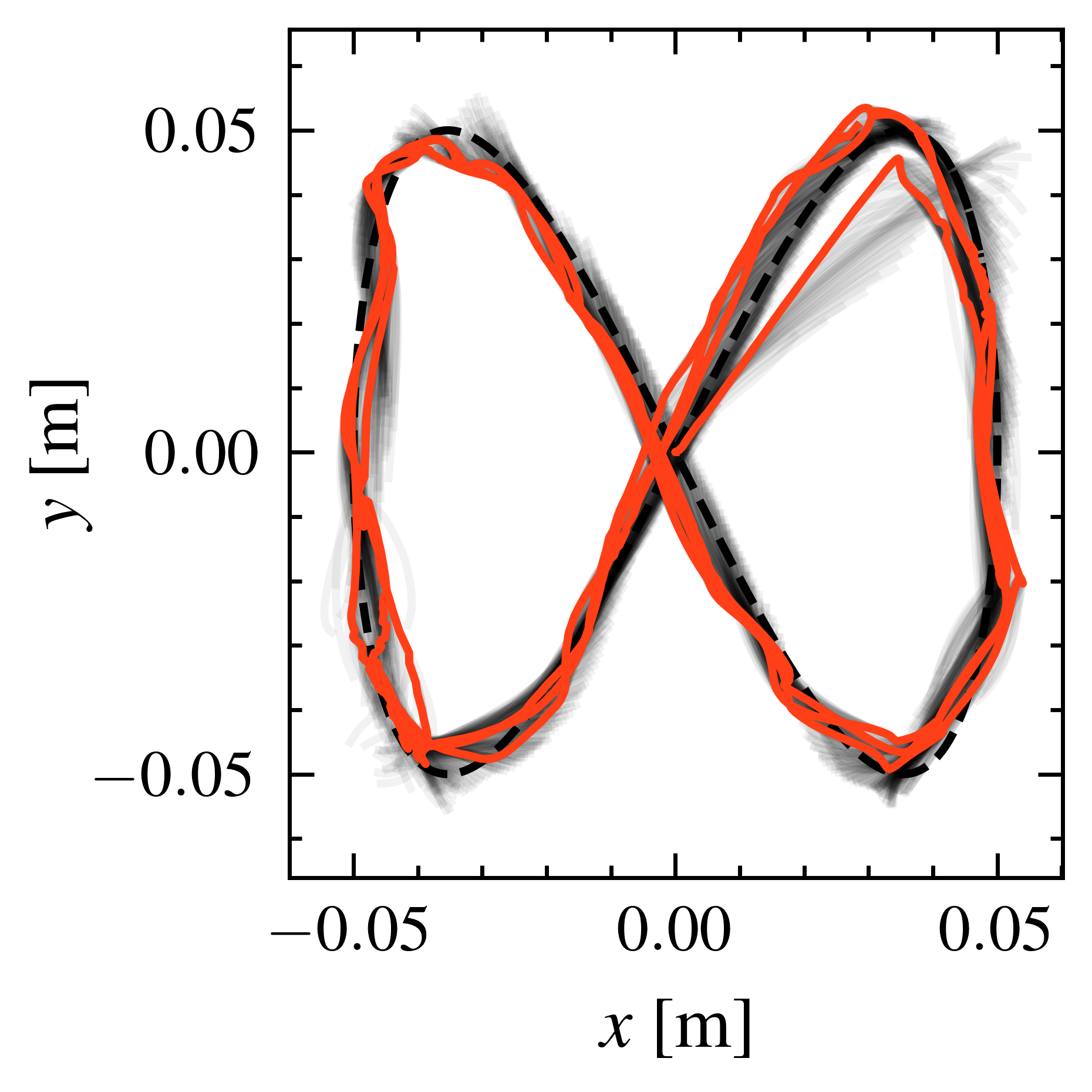}
        \caption{GNN}
        \label{fig:gnn-eight}
    \end{subfigure}

    \caption{Reference tracking of a circle and figure-eight using (a) Koopman Operator modeling, (b) \gls{ssm} orthogonal reduction and (c) our \gls{gnn}-\gls{mpc} approach. The reference is shown in dashed black, while the closed-loop trajectories are shown in red. We additionally provide the \gls{mpc} plans for each time step in grey.}
    \label{fig:tracking-experiments}
\end{figure*}

We note that the linearization step represents the smallest contribution to the overall solve time, indicating that efficient solution algorithms of the \gls{ocp} are crucial for high-dimensional systems. The \gls{qp} solve time remains almost constant, since the number of inputs and constraints remains unchanged, while condensing captures most of the computational burden, which is effectively parallelized on the \gls{gpu}.

\subsection{Hardware experiments}
We validate the proposed \gls{gnn}-\gls{mpc} framework on a physical soft trunk robot. Two tasks are considered: (i) trajectory tracking of a figure-eight and a circle, and (ii) obstacle avoidance. The \gls{mpc} runs at 100~Hz with a prediction horizon of $N=20$. The robot is represented by $M=4$ nodes, whose positions are measured via motion capture, each interacting with $d=2$ neighbors. Each node is described by $\bar n_x = 6$ states, corresponding to its position and velocity in $\mathbb{R}^3$. The \gls{gnn} functions $\psi$ and $\phi$ are parameterized by $4$ and $8$ hidden layers with $50$ and $250$ hidden units, respectively. The cost weights are tuned separately for each model to ensure the best closed-loop performance. Figure~\ref{fig:tracking-experiments} illustrates the closed-loop end effector trajectories for \gls{mpc} involving the Koopman method, \gls{ssm} orthogonal reduction and our approach. We note that \gls{ssm} orthogonal offered better closed-loop results than the optimal linear projections variant on hardware. The \gls{gnn}-\gls{mpc} achieves average tracking errors of $6.25$~mm (figure-eight) and $3.67$~mm (circle), outperforming the baselines by at least $63.6\%$. For the Koopman method, we obtained average errors of $23.76$~mm (figure-eight) and $19.60$~mm (circle), and $17.15$~mm (figure-eight) and $15.77$~mm (circle) for the \gls{ssm} baseline.
These results highlight the ability of the \gls{gnn} to capture hardware-specific effects such as input nonlinearities and tendon slackness. Furthermore, the average solve times are $332$~ms (Koopman), $2.9$~ms (\gls{ssm}), and $9.11$~ms (ours).

\begin{figure}[h!]
\centerline{\includegraphics[width=0.7\columnwidth]{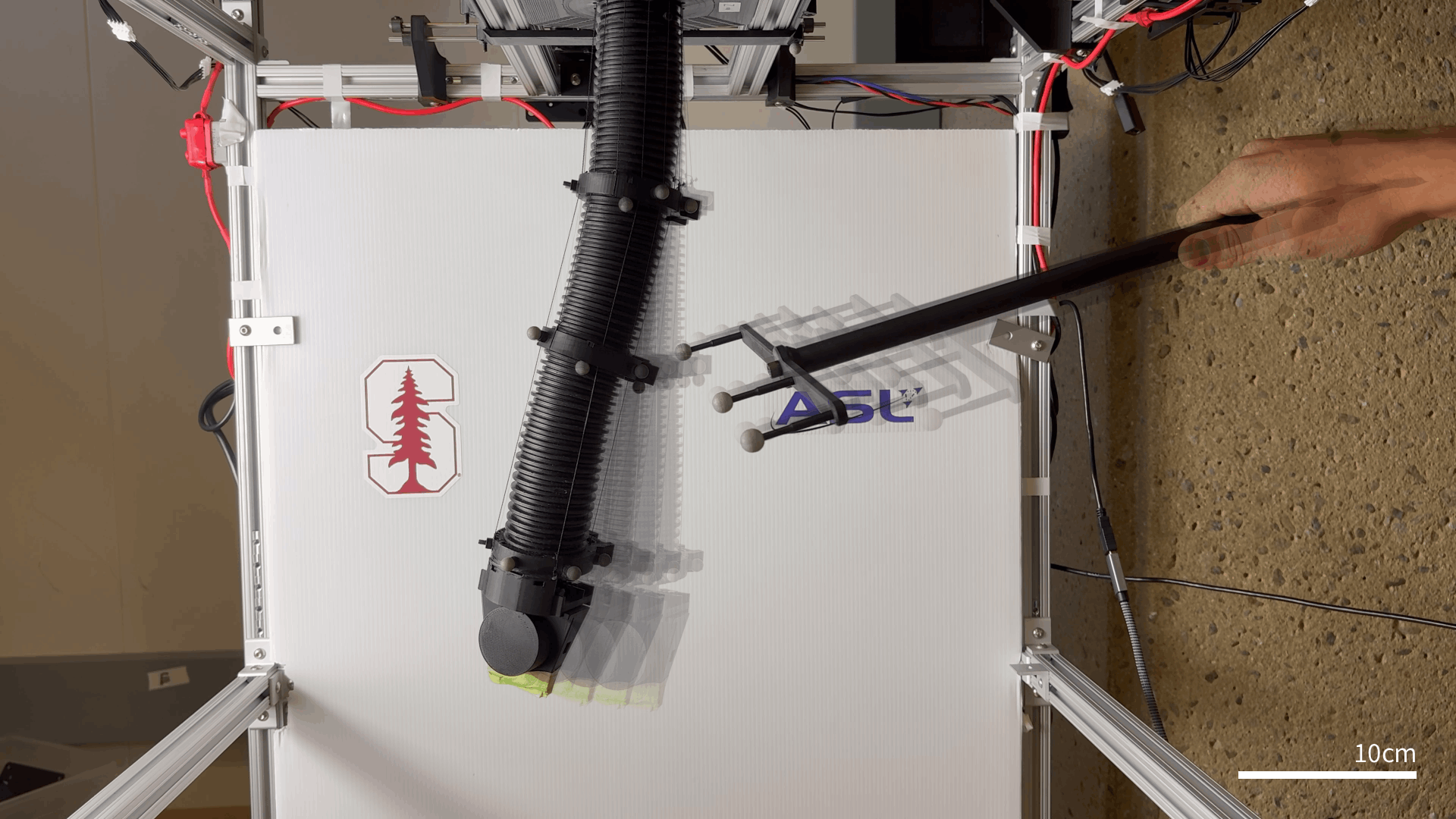}}
\caption{Deflection of the trunk as the obstacle approaches the middle node. The end effector attempts to remain as close as possible to its resting position.}
\label{fig:gnn-obstacle-avoidance-upper-node}
\end{figure}

\begin{figure}[h!]
\centerline{\includegraphics[width=0.7\columnwidth]{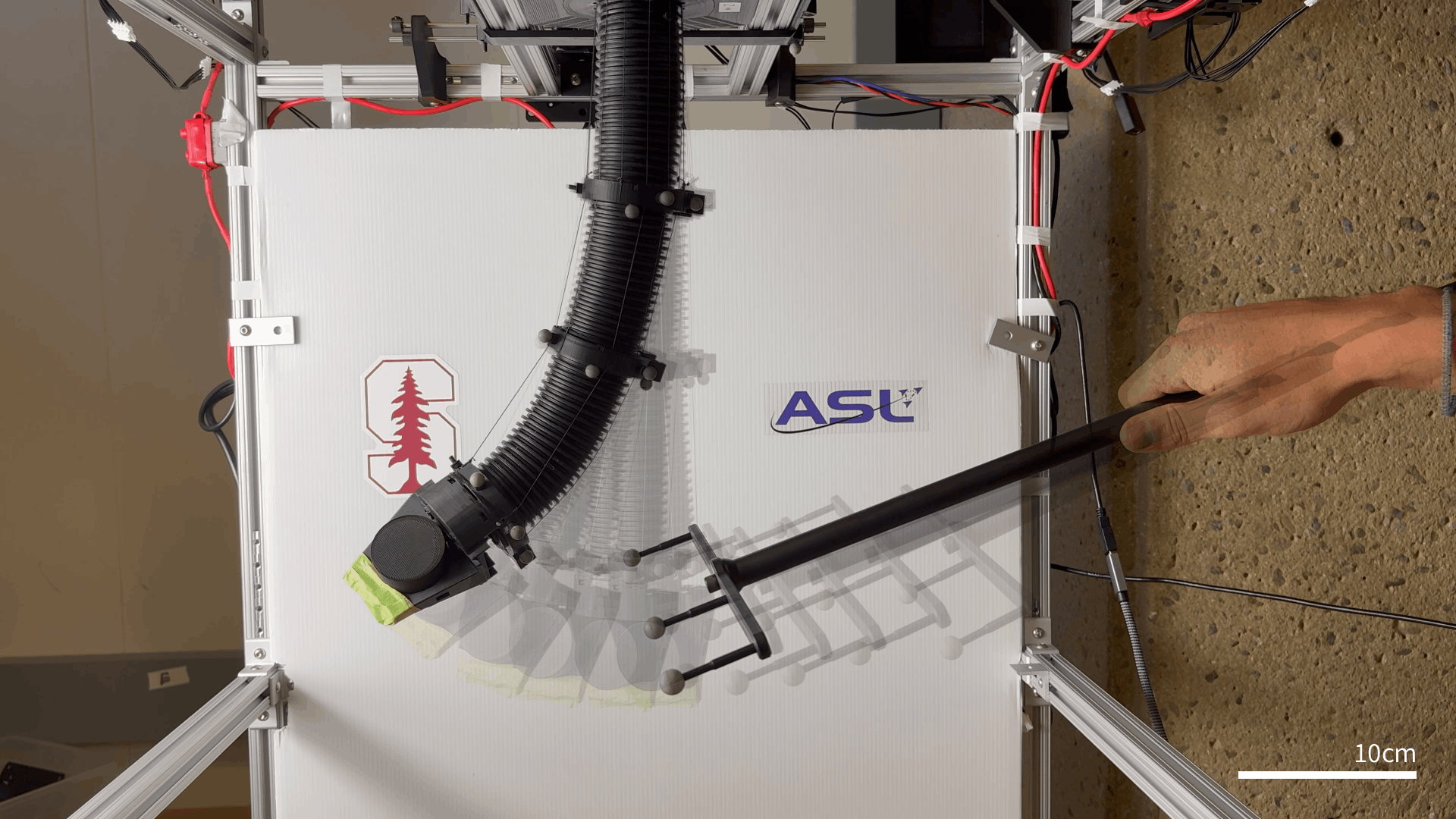}}
\caption{Deflection of the trunk as the obstacle approaches the end effector node. The latter moves further away from the obstacle compared to the previous scenario.}
\label{fig:gnn-obstacle-avoidance-lower-node}
\end{figure}

To address system resonances induced by control delays, we apply a low-pass filter with time-constant $T_f = 0.24s$ to the control inputs generated by the \gls{gnn}. The induced delay is compensated by shifting the reference trajectory by $19$ steps. Further, all considered models require careful tuning and may exhibit unstable behavior depending on the training data and regularization parameters. The baseline models are particularly sensitive to the choice of training trajectories, showing stable predictions for quasi-periodic trajectories in case of \gls{ssm} and random-walk trajectories for the Koopman method. In contrast, the \gls{gnn} demonstrates the ability to learn generalized dynamics from both types, yielding the best results on random-walk trajectories on hardware. We use $50$ training trajectories of 20s each for all models and apply no filtering other than the removal of outliers caused by the motion capture system. Finally, we note that the \gls{gnn} still requires hyperparameter tuning through grid search to ensure smooth and reliable closed-loop control.

Finally, we present the benefit of full-node control by enabling collision avoidance across the entire soft robot through individual barrier terms for each of the four measured nodes in the cost function. Figure~\ref{fig:gnn-obstacle-avoidance-upper-node} shows the deflection of the trunk as the obstacle approaches the middle node, while Figure~\ref{fig:gnn-obstacle-avoidance-lower-node} shows the deformation when the obstacle approaches the end effector node. The \gls{gnn}-\gls{mpc} successfully avoids collisions while keeping other nodes close to their resting position, demonstrating the capability to handle complex interactions with the environment.

\section{Conclusion}
\label{sec:conclusion}

We presented a framework for real-time optimal control of high-dimensional systems that integrates Graph Neural Networks within Model Predictive Control. By modeling dynamics as spatial graphs under the assumption of localized interactions, our approach preserves structural sparsity and enables efficient control through a custom condensing algorithm that scales linearly with the number of system nodes, further accelerated by \gls{gpu} parallelization. In simulation, the \gls{gnn} sustains closed-loop control at 100~Hz for systems with up to 1,000 nodes. While the \gls{gnn} shows slightly lower prediction accuracy than the Koopman operator in simulation, it achieves $63.6$\% better tracking performance on the considered hardware platform. The high dimensionality of the Koopman lifting~\cite{bruder_modeling_2019} further leads to larger solve times, hindering high-frequency real-time control.

Moreover, our \gls{gnn} model is capable of learning the system dynamics from $50$ arbitrary random-walk open-loop trajectories, streamlining the training process compared to the multi-step identification procedures required by the Koopman and \gls{ssm} baselines. Our method further enables applications where control of the entire system state is explicitly considered.

Avenues for future work include extending the framework to time-varying graph structures, improving computational efficiency by addressing constraint handling bottlenecks with advanced solvers, exploring applications beyond soft robotics, and establishing formal stability and robustness guarantees for the closed-loop system.
\section{Acknowledgments}
\label{sec:acknowledgment}

The authors warmly thank Paul Wolff, Carmen Amo Alonso and Lukas Schroth. This work was partially supported by Toyota Motor Engineering \& Manufacturing North America (TEMA). The views expressed in this paper are solely those of the authors and do not necessarily reflect those of the supporting entity.

%TODO: \addtolength{\textheight}{-12cm}

\balance
\bibliographystyle{ieeetr}
\bibliography{references} 

\end{document}